\documentclass[twoside,10pt]{article}
\PassOptionsToPackage{square}{natbib}
\usepackage{Header/jmlr2e}
\usepackage[utf8]{inputenc} 
\usepackage[T1]{fontenc}    

\usepackage{hyperref}       
\usepackage{url}            
\usepackage{booktabs}       
\usepackage{amsmath}
\usepackage{amssymb}
\usepackage{amsfonts}
\usepackage{amsthm}         
\usepackage{bm}
\usepackage{mathtools}
\usepackage{thmtools}
\usepackage{thm-restate}
\usepackage{nicefrac}       
\usepackage{microtype}      
\usepackage{wrapfig}
\usepackage{multirow}
\usepackage{adjustbox}
\usepackage{natbib}
\usepackage{cleveref}
\usepackage{enumitem}
\usepackage{xcolor}
\usepackage{xspace}
\usepackage{subcaption}
\usepackage{stmaryrd}
\usepackage{float}
\usepackage[version=4]{mhchem}
\usepackage[frak=pxtx]{mathalpha}
\usepackage{mathrsfs}
\usepackage{algorithm}
\usepackage{algorithmic}
\usepackage{diagbox}
\usepackage{tikz}
\usetikzlibrary{calc,patterns,angles,quotes}

\Crefname{equation}{Eq.}{Eqs.}
\Crefname{figure}{Fig.}{Fig.}
\Crefname{tabular}{Tab.}{Tabs.}
\Crefname{table}{Tab.}{Tabs.}
\Crefname{section}{Sec.}{Sec.}
\Crefname{appendix}{App.}{App.}

\newtheorem{lemma}{Lemma}

\DeclareMathOperator*{\argmax}{arg\,max}
\DeclareMathOperator*{\argmin}{arg\,min}

\akbcheading
\ShortHeadings{Self-Regulated Agents and Continuous Homeostatic Reinforcement Learning}{Laurençon, Ségerie, Lussange, Gutkin}
\finalcopy

\begin{document}

\title{\vspace{-9mm}Continuous Homeostatic Reinforcement Learning for Self-Regulated Autonomous Agents}

\author{\name Hugo Laurençon(1,2) \thanks{Work done during an internship at Laboratoire de Neurosciences Cognitives, École normale supérieure, Paris.} \email hugo.laurencon@gmail.com \\
       \name Charbel-Raphaël Ségerie(2)  \email crsegerie@gmail.com \\
       \name Johann Lussange(1)  \email johann.lussange@ens.fr \\
       \name Boris S. Gutkin(1,3) \email boris.gutkin@ens.fr \\
      \addr (1) Group for Neural Theory, Laboratoire des Neurosciences Cognitives et Computationelles (LNC2) ISENR U960, Ecole Normale Supérieure,\\
      Paris, France \\
      \addr (2) Ecole Normale Supérieure Paris-Saclay,\\
      Université Paris-Saclay,\\
      Paris, France \\
      \addr (3) Center for Cognition and Decision Making, \\ 
      Institute of Cognitive Neuroscience, \\
      NRU Higher School of Economics, Moscow, Russia}
\maketitle

\begin{abstract}

Homeostasis is a prevalent process by which living beings maintain their internal milieu around optimal levels. Multiple lines of evidence suggest that living beings learn to act to predicatively ensure homeostasis (allostasis). A classical theory for such regulation is drive reduction, where a function of the difference between the current and the optimal internal state. The recently introduced homeostatic regulated reinforcement learning theory (HRRL), by defining within the framework of reinforcement learning a reward function based on the internal state of the agent, makes the link between the theories of drive reduction and reinforcement learning. The HRRL makes it possible to explain multiple eating disorders. However, the lack of continuous change in the internal state of the agent with the discrete-time modeling has been so far a key shortcoming of the HRRL theory. Here, we propose an extension of the homeostatic reinforcement learning theory to a continuous environment in space and time, while maintaining the validity of the theoretical results and the behaviors explained by the model in discrete time. Inspired by the self-regulating mechanisms abundantly present in biology, we also introduce a model for the dynamics of the agent internal state, requiring the agent to continuously take actions to maintain homeostasis. Based on the Hamilton-Jacobi-Bellman equation and function approximation with neural networks, we derive a numerical scheme allowing the agent to learn directly how its internal mechanism works, and to choose appropriate action policies via reinforcement learning and an appropriate exploration of the environment. Our numerical experiments show that the agent does indeed learn to behave in a way that is beneficial to its survival in the environment, making our framework promising for modeling animal dynamics and decision-making.

\end{abstract}

\section{Introduction}

Reinforcement learning (RL) has been of particular interest in recent years in machine learning (ML) and artificial intelligence. Dramatic advances have been made (\citep{Mnih} \citep{Silver}, \cite{Silver2018}), particularly due to progress in deep learning (\citep{Krizhevsky}). The general framework of RL (\citep{Sutton}), which studies how an agent optimizes its decision-making in an environment, is easily applicable to many fields, such as economics (\citep{Lussange}), psychology (\citep{Shteingart}), control theory (\citep{Kretchmar}) or neuroscience (\citep{Niv}).

Arguably, the differences in how RL has been applied across the different disciplines, is influenced by the major goal: in neuroscience research, one seeks to explain how the brain carries out cognitive tasks, while artificial intelligence tries to design models that can solve a fixed task. Yet, both of these fields have started to intermingle, with RL modeling being a main area of advance (\citep{Kriegeskorte}, \citep{Richards}). Indeed, current ML models propose a framework for neuroscience (\citep{Cichy}, \citep{Guclu}, \citep{Kietzmann}) and make it possible to do reverse engineering (\citep{Dubreuil}). Conversely, the observation of human behaviors to create brain-inspired architectures often leads to better performance, as with the attention mechanism (\citep{Vaswani}) or with the lifelong learning (\cite{Tessler}). For example, it was recently found that distributional RL, which produced better results than the classical approach (\citep{Bellemare}), is actually performed in the brain (\citep{Dabney}).

An important and challenging point relevant for neuroscience is therefore to propose a model of an agent with characteristics and behaviors compatible with biological and psychological data and aspects of real beings (e.g. rodents, primates, humans). For example, an agent that has both physiological and psychological needs, and an ability to learn (an agent needs to eat, sleep, consumes energy in actions and is driven by motivation to act in the world). This agent would thus evolve in an environment where it has internal needs, and would choose its actions to fulfill them. One prominent approach for such bio-mimetic agents is to ground them in homeostatic and allostatic regulation principles. Following this framework, \citeauthor{Man} have conceptually defined a class of robots based on homeostasis for living. Robots would have emotional qualities (feelings, emotions, affect) directly motivated by their internal states, would operate in a vast environment and not only on a very specific task, and would not be prepared for all eventualities but rather would learn directly from interacting with their environment and observing their internal states. The RL theory, in which an agent chooses its actions in order to maximize its rewards in its environment, proposes a framework adapted to this modeling. Along with this, and also for a behavioral study of animals, the homeostatic regulation theory provides a framework where an agent receives negative feedbacks as soon as its internal state deviates from its homeostatic set point, i.e. its ideal point (\citep{Staddon}, \citep{Toates}). This second theory therefore proposes a model where an agent would behave in such a way as to minimize internal deviations from its homeostatic set point, as opposed to the RL theory where one seeks to maximize its external rewards.

\citeauthor{Keramati} linked the two theories by defining a reward function based on the internal state of the agent, and by showing that under this model, the problem of maximizing reward and the problem of minimizing deviations from the homeostatic set point are in fact equivalent. This homeostatically regulated reinforcement learning (HRRL) theory has made it possible to model primitive behaviors in agents, such as the consumption of different resources, but also many more evolved behaviors, such as risk aversion, alcohol tolerance, cocaine addiction or anticipatory control (\citep{Keramati2014}, \citep{KeramatiThesis}). However, there are limitations to this theory. For example, the internal state of the agent is assumed to be fixed when the agent does not take action, thus, the agent only has to reach its homeostatic set point once and take no more action to be in its perfect state. This is obviously not the case in practice, where if at one moment the agent has an optimal amount of a certain resource in its body, it will still need to consume it again a few hours later, in order to balance the consumption of this amount by its body to live. The agent thus performs the process of allostasis, which consists of achieving stability or homeostasis through physiological or behavioral change (\cite{Ramsay}). Moreover, in \cite{Keramati}, the actions taken by the agent are necessarily carried out at discrete and regular time steps, which results in a discount factor that does not model the time in the temporal preference of the present over the future, but only models the preference to carry out one action before another, without any notion of temporal relationship (e.g. the delay) between the two actions.

In this paper, we propose an extension of this HRRL theory in continuous time by including a dynamic to the agent. The main \textbf{contributions} are the following:

\begin{itemize}[leftmargin=12mm, itemsep=0mm, partopsep=0pt,parsep=0pt]
    \item \textbf{Agent embodiment and internal state dynamics:} We take into account the dynamics of the internal state of the agent, based on a self-regulating effect of the body state. We consider a continuous change of internal state for the agent even if no action is taken. Therefore, the agent has a constant need to regulate itself by choosing the actions that bring it the rewards it seeks. We call agents verifying this type of dynamic self-regulated agents. We also take into account the agent's internal energy levels, which need to be replenished by periods of rest that are caused by states of excessive fatigue (during which actions cannot be performed).
    \item \textbf{Continuous time implementation:} We get closer to reality by proposing a continuous-time model, where the agent evolves and moves in a continuous environment in space. We show that the theoretical results present on the discrete model remain true in continuous time. The continuity in time also allows for a better compactness of the model.
    \item \textbf{Self-learning of the agent in interaction with the environment:} In numerous previous studies, the lack of dynamics and discrete time limited the realism of numerical simulations in complex environments. Here, we propose a numerical scheme allowing the agent to discover and explore its environment, and to learn on its own how its internal state ("the body") functions and reacts to the different actions it takes. Using this information that is learned over time, the agent can constantly update its policy towards actions that will ensure the most future rewards. This policy depends on its internal state, but also on the structure of the environment, which was not the case before.
\end{itemize}

\subsection{Background and Related Work}

In this work, we consider the dynamic interactions between the agent's internal state and the environment in which the agent acts through reinforcement learning of optimal reward-seeking behaviors. There have been several previous approaches to incorporating internal states into learning and motivation. We briefly list a few examples, starting with classical drive reduction theory to Homeostatically Regulated Reinforcement Learning (HRRL).

\textbf{Drive reduction theory:} Models based on homeostatic regulation assume that the behavior of animals is a function of their internal states. This implies that animals would have representations of physiological variables derived from their organic needs, and that each of these variables would have an ideal value, called the homeostatic set point, for the proper functioning of the organism. In an animal, the more the value of a variable deviates from its homeostatic set point, the more the animal would be in a negative affective state and this negative deviation would produce a motivational drive to correct the need. Thus, these models, mathematically related to control theory, are also called negative feedback models because they are based on the regularity of the variables describing the internal state of an animal, returning a punishment when this regularity is no longer respected. The drive of an animal is defined as a measure of the deviation of its internal state from its homeostatic set point. Heuristically, it is therefore a measure of the animal's discomfort or negative affect. The theory of drive reduction proposed by Hull (\cite{Hull}) assumes that an animal chooses its actions to reduce its drive. However, although it has had some success in explaining certain behaviors (\citep{Staddon}, \citep{Toates}), this theory has a number of major drawbacks that have limited its applicability: For example, it cannot explain anticipatory consumption or anticipatory response, when an animal does not deviate from its homeostatic set point but still consumes a resource in anticipation of a future need or produces a response in anticipation of a future homeostatic challenge (\cite{Wingfield}). To understand this type of behavior, it is not enough to observe the internal state of the animal, it is also necessary to model the environment and the amount of resources present, as well as the animal's ease of access to these resources and, last but not least, the animal's predictive learning processes. 

\textbf{Hullian drives in reinforcement learning:} There has been at least one previous attempt to incorporate the concept of Hullian drive into reinforcement learning algorithms. For example, \citeauthor{Konidaris} introduced an RL model whose reward function is a sum of Hullian drive weighted by time-dependent coefficients indicating the priority of the drive. A Hullian drive is a drive varying between 0 for total dissatisfaction and 1 for total satiation. In \citeauthor{Konidaris}, the authors put priority coefficients to give importance to hard-to-reach resources in the environment. In our model, we choose not to put a priority coefficient because it is then easier to obtain intuitions and theoretical results on certain behaviors. Moreover, thanks to the exploration of its environment, our agent learns directly the quantities available for each resource and adapts its policy accordingly. As in this study, the agent's policy depends on its internal state, but also on its external environment. The authors applied a constant penalty for each move for each drive where a resource is not consumed. We go further with a regulatory effect of an arbitrary function of time and control (which is not a function only of motion as before) that takes into account the correlations between the different drives. In addition, the simulation used the SARSA algorithm, which is not always robust for small-time steps. We address this limitation in this study.

\textbf{Homeostatically Regulated RL (HRRL or HRL):} \citeauthor{Keramati} introduced an RL theory where rewards and punishments are directly derived from internal state deviations and the Hullian drive function. HRL denotes the agent's physiological state $x_t = (x_{1,t}, \dots, x_{N,t})$ with $N$ variables at time $t$ and by $x^{*} = (x_{1}^{*}, \dots, x_{N}^{*})$ the homeostatic set point. The drive function is defined by $d(x_t) = (\sum_{i=1}^N |x_i^{*} - x_{i,t}|^n)^{1/m}$ with $n,m$ integers $\geq 1$, and the reward of the transition from state $h_t$ to $x_{t+1}$ by $d(x_t) - d(x_{t+1})$. In this configuration, Keramati and Gutkin showed, among other things, that maximizing the sum of discounted rewards is equivalent to minimizing the sum of discounted drives, thus producing a suitable reinforcement learning framework related to drive reduction theory. The authors have obtained theoretical results showing certain behaviors, but have not performed numerical simulations with an artificial agent acting in simulated physical environments and learning according to this scheme, which we do in this study by providing the agent with a policy and a learning method.

\section{Preliminaries}

In a general reinforcement learning problem (\citep{Sutton}), an agent, in a certain state, can select an action among those available in its state. This action will take it to another state and give it an immediate reward. It can then choose another action, and so on. Its goal is to find the policy that maximizes the discounted sum of future rewards for each state at time $t$. The value function, for a certain policy, is a function that associates to a state the expected discounted sum of future rewards for an agent starting from this state and following this policy.

Let $n$ be an integer, $\zeta_t \in \mathrm{R}^n$ the state (internal and external) of the agent at time $t$ and $\zeta: t \mapsto \zeta_t$ the trajectory function of the agent's state. We denote the space of possible actions at time $t$ when the agent is in $\zeta_t$ by $\mathcal{A}_{\zeta_t, t}$ and the space of all actions by $\mathcal{A}$. The policy function determining the agent's choices is denoted by $\pi: \mathrm{R}^{n} \mathrm{R} \to \mathcal{A}_{\zeta_t, t}, (\zeta_t, t) \mapsto a$, and the reward received by following this policy at time $t$ by $r(t) = r_{\pi(\zeta_t, t), \zeta_t, t}$. We denote by $\Pi$ the set of all admissible policies. Note that in the deterministic case and for a fixed policy, the initial state $(\zeta_0, t_0)$ and the policy function $\pi$ completely determine the path function $\zeta$ and the reward function $r$. The value function, for an agent in $\zeta_t$ at time $t$ following a policy $\pi$, is defined as follows
\begin{align}
    V^\pi(\zeta_t, t) = \int_{t}^{\infty} \gamma^{s-t} r(s) \, \mathrm{d}s \label{eq:V}
\end{align}
where $\gamma \in ]0,1[$ is the discount factor that accounts for time preference. The optimization problem for the agent is
\begin{align}
    \argmax_{\pi \in \Pi} V^\pi, \label{eq:problem}
\end{align}
where it is trying to find the policy function that maximizes the value function for each state at each time.

\section{Methods}
    
\subsection{Modeling}

\subsubsection{Dynamics and decision-making}

In what follows, variables useful for modeling will be defined, and the dimension of these variables and their physical meanings will depend on the level of abstraction of the model, depending on whether one considers a macroscopic scale, or rather a microscopic scale for example. This level of abstraction must be considered as a parameter in itself of the model that is set. The higher the dimension of the variables (or the lower the scale), the more biologically possible the model is. The lower this level (or the higher the scale), the more efficient the model is in simulations and the more understandable it is. To understand the principle, examples of such scales are given in Table~\ref{tab:ex_scales}.

\begin{table}
  \centering
  \begin{tabular}{|c|c|c|c|c|c|c|c|}
    \hline
    \diagbox{Feature}{Scale}
                   & Macroscopic & Microscopic \\
    \hline
    $x$ & Quantity of food in the body & Concentration of glucose in the body \\
    \hline
    $e$ & Spatial coordinates of the agent & Probability of finding a resource \\
    \hline
  \end{tabular}
  \medskip
  \caption{Examples of features in $x$, the internal state $x$, and $e$, the external environment, for a macroscopic and a microscopic scale. Only features in $x$ are homeostatic.}
  \label{tab:ex_scales}
\end{table}

We consider a framework in which an agent evolves in a world in which its motivations and decision processes depend on the perception of its internal state as well as on the perception of its external environment. Let $x \in \mathrm{R}^{n_{int}}$ be a vector describing its internal state where each feature needs to be regulated, and let $x^*$ denote its homeostatic set point. Assume that each feature of $x$ is bounded. This is a constraint from the agent's embodiment, which will not be valid if a feature is too small or too large, and will either be regulated automatically by the organism or will cause the agent's death. We further propose that the internal information the agent has access to are the differences between the individual internal variables and their respective set points: $\delta := x - x^*$ (where the order of the difference is arbitrarily chosen and can be evaluated as RL comparison rewards (\cite{Matignon}). For example, the agent may encode a hunger signal at a given time, but not exactly the optimal amounts of its bodily needs, nor its current state. We also define the external environment of the agent within its view-field $e \in \mathrm{R}^{n_{ext}}$, and the entire world $\zeta = [\delta^{T}, e^{T}]^{T}$. At any time $t$, the agent in $\zeta_t$ can perform an action $a \in \mathcal{A}_{\zeta_t, t}$. The agent is limited in its choice by its environment (for example by the place in which it is) and by its internal state and time (because some actions depend on energy). The action taken, in turn, will have a consequence (a control $u \in \mathrm{R}^{n_{int} + n_{ext}}$) for its internal state and its environment. We assume that from the agent's point of view, the dynamics of $\zeta$ is described by an equation of the form
\begin{equation}
    d \zeta = f(\zeta, u, t)dt + g(\zeta, t)dS \label{eq:main}
\end{equation}
where $f, g$ are functions and $S$ is a stochastic process. $f$, $g$ and $S$ are unknown to the agent at the beginning of its task: it does not have the information of how its body and the external world react and has no estimate of the behavior of the stochastic process $S$.

We can distinguish several potentially overlapping causes of a change in the agent's internal state and environment:
\newline
(a) an unconscious automatic autoregulation of the organism, modeled by the function $f$ and its variable $\zeta$, which can account for internal processes of the agent's body (e.g., animal physiology, robot mechanics, and multi-component interactions). We note that from the perspective of biology, self-regulation is a common physiological process (\cite{Polynikis}, \cite{Pattaranit}). To give some examples, the kidney uses a mechanism called tubuloglomerular feedback to regulate the glomerular filtration rate in response to changes in sodium concentration (\cite{Versypt}, \cite{Thomson}). Autoregulation of cerebral blood flow has been demonstrated in the presence of \ce{CO_2} (\cite{Panerai}). This change in $\zeta$ takes place without the agent taking any action of its own. If the agent does not perform any action, its internal state will naturally deviate from its current state;
\newline
(b) a control that the agent exercises and that has an impact on its environment and its internal state, modeled by the function $f$ and its variable $u$. For example, the action of moving, which modifies the environment while requiring an effort, and which thus has an impact on the internal state;
\newline
(c) the time that changes and modifies the external environment and the internal state, modeled by the function $f$ and its variable $t$. For example, the time of day and the current season will drive the temperature and brightness of the environment in a certain direction. Time can also model old age, by progressively modifying the function $f$, and thus the way the body reacts over time;
\newline

(d) a stochastic control that the agent undergoes, leading to an unexpected change mostly in the environment, modeled by the $g$ function and the stochastic process $S$. Stochasticity intervenes in everything that the agent cannot control, in particular the behavior of other agents around it or the weather.

The control, by changing the environment, has a direct impact on the actions that can be taken in the future, since these will depend on the new environment. By changing the internal state, it also has a direct impact on the agent's drive $d$. Control brings the agent to a more or less comfortable state, depending on whether it is moving towards or away from its homeostatic set point $x^*$, or equivalently whether the drive is decreasing or increasing. The agent's goal is to minimize its drive on its task by finding the optimal policy that will allow it to take good actions in order to stay as close as possible to its homeostatic setpoint. The deviation function $J^\pi : \mathrm{R}^{n_{int} + n_{ext}} \times \mathrm{R} \to \mathrm{R}$ for an admissible policy $\pi$, which represents the integral of the agent's discounted drive over its remaining lifetime, is given by
\begin{equation}
    J^\pi(\zeta_t, t) = \mathbb{E} \left ( \int_{t}^{\infty} \gamma^{s-t} d(\delta(s)) \mathrm{d}s \right ) \ \label{eq:J}
\end{equation}
where $\zeta$ follows equation \eqref{eq:main} (and thus $\delta$ depends on $\pi$) on $[t,+\infty[$ with the initial condition $\zeta(t) = \zeta_t$, the control function $u$ satisfies $\forall s \in [t,+\infty[, u(s) = u_{\pi(\zeta(s),s)}$ (we will say that the control function is associated with the policy if it meets this last condition) and the expected value is here because of the stochasticity in \eqref{eq:main}. Note that the integral is well-defined thanks to the discount factor and the fact that $x$ and $x^*$ are bounded. Concretely, the value of the deviation function $J^\pi(\zeta_t, t)$ indicates how bad it is for the agent to follow the policy $\pi$, starting from the state $\zeta_t$ at time $t$. The problem of the agent is thus
\begin{equation}
    \argmin_{\pi \in \Pi} J^\pi \label{eq:J_problem}
\end{equation}
with the same conditions as before.

\subsubsection{Modeling in practice}

The modeling framework described above is very general and can model many situations. In practice, in order to perform numerically efficient modeling or to obtain theoretical results more easily, one can use the general model as a starting point and perform two distinct types of modifications. The first type of modification involves reducing the generality of the model by making simplifying assumptions. The second type of modification concerns the choice of scale and the explicit definition of the variables involved, the possible actions, etc. We give in this section the simplifying assumptions we have made to reduce the generality of the model (first type of modification) for the purpose of this report.

For the dynamics, we assume that $x^{*}$ is constant, that there is no stochasticity and that $f$ is continuous. For decision-making, we assume that the actions are sufficiently elementary that there is no time interval during which the agent, after performing one action, cannot choose another. Here, the agent can change action at any time. The control resulting from an action is assumed to be deterministic, i.e. for the same chosen action, the resulting control will always be the same. We assume that the control function $u$ is piecewise continuous. The set of admissible actions $\mathcal{A}_{\zeta_t}$ and the policy $\pi$ do not depend on the variable $t$. The action space is discrete, but the state space is continuous.

\subsection{Theoretical results}

\subsubsection{An equivalent formulation of the optimization problem}

We define the reward at time $t$ for an agent whose internal state follows the trajectory function $\zeta = [\delta^{T}, e^{T}]^{T}$ as follows 
\begin{equation}
    r(t) = -(d(\delta))'(t).
    \label{eq:r}
\end{equation}
Intuitively, the reward, which can be positive or negative, is proportional to the variation of the drive of the agent, and thus to what the agent has gained (or lost) in comfort with respect to the stasis point between time $t$ and $t+dt$. This variation of the drive is implied by the control and thus by the action that the agent has taken at time $t$.

\begin{lemma}
The pursuit of homeostatic stability is equivalent to the maximization of the reward. Formally, we have
\begin{equation}
    \argmax_{\pi \in \Pi} V^\pi = \argmin_{\pi \in \Pi} J^\pi.\label{eq:equivalence_second_problem}
\end{equation}
\end{lemma}

\begin{proof}
We can notice by doing an integration by parts (valid even in the case where the function $\zeta$ is continuous everywhere and piece-wise $\mathscr C^1$, which is the case when $f$ is continuous and $u$ is piece-wise continuous) that we have
\begin{equation}
    V^\pi(\zeta_t, t) = d(\delta_t) + \ln(\gamma) J^\pi(\zeta_t, t) \label{eq:second_prob}
\end{equation}
with $ln(\gamma) < 0$. We can then conclude the proof.
\end{proof}

We have reformulated the problem in an equivalent way using the classical variables of reinforcement learning, which are the reward and the value function. This property has already been proved in the discrete case (\cite{Keramati}). It establishes a link between the maximization of the integral of the discounted rewards and the minimization of the integral of the discounted drive.

\subsubsection{Properties of the reward and the drive function}

In this section, we take the derivative of the reward function with respect to several quantities (realized in discrete time in Keramati), and study the sign to show the underlying properties of this function, reflecting behaviors in the agent.

We define the drive as
\begin{equation}
    d(\delta)= \sqrt{\delta^T \delta}
\end{equation}
(in practice, $\sqrt{\epsilon + \delta^T \delta}$ to take the derivative in $0$). The reward at time $t$ is therefore
\begin{equation}
    r(t) = -(d(\delta))'(t) = - \delta_{t}^{T} \dot{\delta_{t}} / \sqrt{\delta_{t}^{T} \delta_{t}}.
\end{equation}

We consider a situation in which an agent starts at time $t_0 = 0$ with a state $\delta_0 = [\delta_{0,1}, \delta_{0,2}, \dots]^T$, where $\delta_{0,1}$ and $\delta_{0,2}$ represent the levels of the agent's first two needs. From $t_0$ onwards, the agent continuously consumes the same resource which gives it a control $u = [m, 0, \dots, 0]$ constant in time, with $m$ the quantity of resource consumed per unit of time. Let us consider a time $t$ sufficiently close to $t_0$ so that the regulating effect of the body is negligible compared to the quantity of resource ingested. We have $\delta_t = \delta_0 + tu$ and the drive and the reward at time $t$ are

\begin{align}
  & d(t) = \sqrt{t^2m^2 + 2tm\delta_{0,1} + \delta_{0}^{T}\delta_{0}} \text{,} \\
  & r(t) = - \frac{(\delta_{0,1} + tm)m}{\sqrt{t^2m^2 + 2tm\delta_{0,1} + \delta_{0}^{T}\delta_{0}}}.
  \label{eq:drive_reward_derivative}
\end{align}

\textbf{Effects of deviation from the homeostatic set point for the feature receives an outcome:} Taking the derivative of the reward with respect to$|\delta_{0,1}|$, we find that

\begin{equation}
  \frac{\partial r(t)}{\partial |\delta_{0,1}|} \left\{
  \begin{aligned}
  \leq 0 & \text{ if } \delta_{0,1} \geq 0  \\ 
  \geq 0 & \text{ if } \delta_{0,1} \leq 0  \\ 
  \end{aligned}
\right.
.
\label{derivee_1}
\end{equation}

The first case means that if an agent has exceeded its homeostatic set point for a need, and it continues to consume a resource affecting this need, it will receive a punishment (negative reward) that is greater the higher its deviation from the homeostatic set point was initially. The second case means that for an agent deprived of a resource, a fixed amount consumed of that resource will have a greater motivational outcome if the agent's initial need for the resource was high rather than low, as observed in \cite{Hodos}.

\textbf{Cross need interactions, effects of deviation from the homeostatic set point for a feature that does not receive an outcome:} Taking the derivative of the reward with respect to $|\delta_{0,2}|$, we find that

\begin{equation}
  \frac{\partial r(t)}{\partial |\delta_{0,2}|} \left\{
  \begin{aligned}
  \leq 0 & \text{ if } \delta_{0,1} + tm \leq 0  \\ 
  \geq 0 & \text{ if } \delta_{0,1} + tm \geq 0   \\ 
  \end{aligned}
\right.
.
\label{derivee_2}
\end{equation}

In the first situation, $\delta_{0,1} + tm \leq 0$, so $x_{t,1} = x_{0,1} + tm \leq x^{*}_{1}$ and the agent is still below its homeostatic set point for the first need at time $t$. The agent will gain a positive reward by consuming the resource affecting her first need, but the negative derivative means that this reward will be reduced if $|\delta_{0,2}|$ increases. The interpretation of the second situation is similar, but the reward is now negative, since the agent has exceeded its homeostatic set point for the first need. Such inhibitory effects occur in nature, as shown experimentally by \cite{Dickinson}, with, for example, food deprivation tending to suppress water-related responses.

\textbf{Effects of resource dose:} Taking the derivative of drive with respect to $tm$, which is the amount of resource consumed at time $t$, we find that

\begin{equation}
  \frac{\partial d(t)}{\partial tm} \left\{
  \begin{aligned}
  \leq 0 & \text{ if } \delta_{0,1} + tm \leq 0  \\ 
  \geq 0 & \text{ if } \delta_{0,1} + tm \geq 0   \\ 
  \end{aligned}
\right.
.
\label{derivee_3}
\end{equation}

This means that if an agent has not (resp. has) reached its homeostatic setpoint for a need at time $t$, its training would have been smaller, and thus the agent would have been closer to its homeostatic setpoint, if the amount of resource consumed was larger (resp. smaller), as shown for rats in \cite{Skjoldager}.

\subsubsection{Hamilton-Jacobi-Bellman equation}

In reinforcement learning, the optimal value function and an optimal policy are often derived from a Bellman equation. \citeauthor{Doya} showed a continuous time extension with a discounted Hamilton-Jacobi-Bellman (HJB) equation on the value function $V$. By performing the same proof, we can have the equation on $J$ instead of $V$, which allows us to have one less derivative by replacing the reward by the drive:
\begin{equation}
    -\log(\gamma)J^{*}(\zeta_t) = \min_{a \in \mathcal{A}_{\zeta_t}} d(\zeta_t, u_a) + \frac{\partial J^{*}}{\partial \zeta}(\zeta_t) \cdot f(\zeta_t, u_a)
\end{equation}
where $J^{*}$ is the optimal deviation function, $u_a$ is the deterministic control resulting from the action $a$ and $d$ is the drive function, with the conventions that $d(\zeta_t, u_a)$ is the drive of the new state of the agent after performing the action $a$ in $\zeta_t$, and that symbolically $d(\zeta_t) = d(\delta_t)$. The intuition behind this equation is obtained with the optimality principle and by making the analogy with the known discrete Bellman equation (\cite{Sutton}). The $Q$-learning not being robust in the presence of small-time steps (\cite{Tallec}), we rely on this equation to propose our algorithm.

\subsection{Learning algorithm for the agent}

Here we present the Algorithm~\ref{learning_algo} that allows the agent to learn by interacting with its environment (see the graphical representation in Figure~\ref{fig:schema_algo}). The algorithm is based on the principle of policy improvement, where at each step a value function is evaluated, and the policy is updated directly using this value function. We note that classical reinforcement learning heuristics to improve the quality of learning are deliberately not implemented here, as the goal is not to propose the most efficient algorithm possible, but a proof of concept that it is possible to learn for an agent starting from zero knowledge and following a natural and plausible approach to action selection and learning from experience.

\begin{figure}[h]
  \centering
  \includegraphics[scale=0.205]{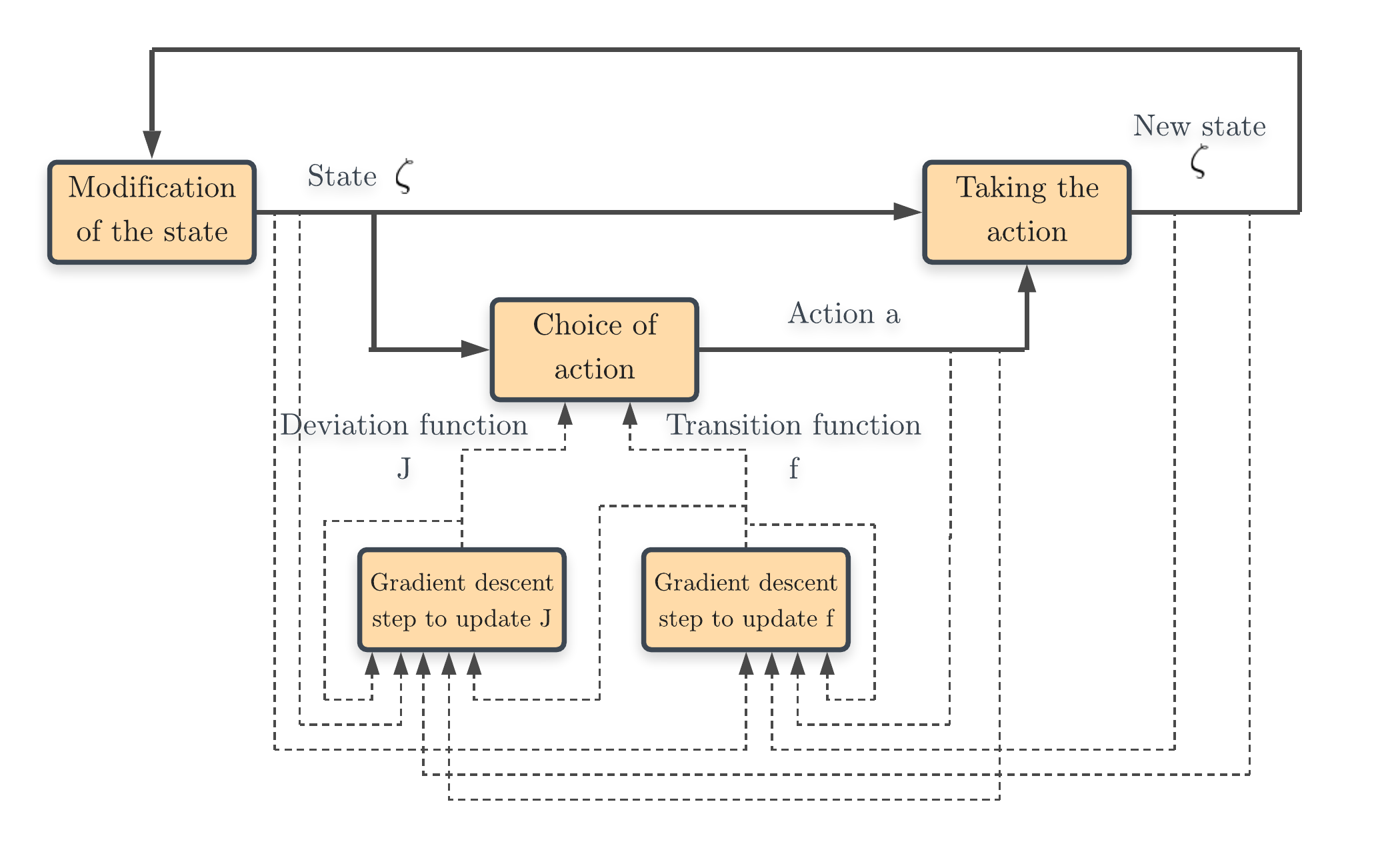}
  \caption{Visual diagram of Algorithm~\ref{learning_algo}. The dotted arrows are relative to the deviation function and the transition function.}
  \label{fig:schema_algo}
\end{figure}

We discretize time by $\Delta_t$ time steps. The discretization in time is necessary here to build the algorithm, but the proposal of a continuous theoretical framework is justified by a better modeling, an economy of notations, or the possibility to make adaptive time steps. In the initial state, the agent does not have access to the functioning of its internal state (the body), represented by the function $f$. Over time, the agent learns to approximate this function through its experiences. We thus have a model-based algorithm, since the transitions between internal states are modeled. On the other hand, the drive function $d$ is known initially, modeling the agent's a priori interoception. We note that, in classical RL, rewards are known (by definition) a priori.

The agent's action is either taken randomly with probability $\epsilon$ to facilitate exploration, or based on the HJB equation and estimates of $J$ and $f$. Note that, for a certain policy $\pi$, the deviation function $J$ is defined exactly with equation \eqref{eq:J}. However, this equation requires the calculation of and integral over the lifetime of the agent following this policy, which is impossible for the agent since it does not have access to the information of its future. Therefore, the agent maintains an estimation of $J$ instead, and update it according to Algorithm \ref{algo:algo_agent}.

The estimated transition and deviation functions are neural networks. The estimated deviation function $J$ is updated at each step by minimizing an associated error (\cite{Doya}). The gradient of the neural networks with respect to the inputs is also computed by backpropagation.

\begin{algorithm}
  \caption{Learning algorithm for the agent \label{algo:algo_agent}}
  \label{learning_algo}
  \begin{algorithmic}
    \STATE Randomly initialize the transition function $f(\zeta, u | \theta^{f})$ and the deviation function $J(\zeta | \theta^J)$ with weights $\theta^{f}$ and $\theta^{J}$
    \STATE Receive initial observation state $\zeta_1$
    \FOR{$k = 1,\dots, K$}
      \STATE With probability $\epsilon$ select a random action $a_k \in \mathcal{A}_{\zeta_k}$, otherwise select
      \begin{equation*}
          a_k = \argmin_{a \in \mathcal{A}_{\zeta_k}} d(\zeta_k + f(\zeta_k, u_{a} | \theta^{f})\Delta_t) + \frac{\partial J}{\partial \zeta}(\zeta_k | \theta^J) \cdot f(\zeta_k, u_a | \theta^{f})
      \end{equation*}
      \STATE Execute action $a_k$ and observe new state $\zeta_{k+1}$
      \STATE Update the transition function and the deviation function by performing a gradient descent step on
      \begin{align*}
        & L_f = (\zeta_{k+1} - \zeta_k - f(\zeta_k, u_{a_k} | \theta^{f})\Delta_t)^T (\zeta_{k+1} - \zeta_k - f(\zeta_k, u_{a_k} | \theta^{f})\Delta_t) \text{ with respect to } \theta^{f} \\
        & L_J = (d(\zeta_{k+1}) + \frac{\partial J}{\partial \zeta}(\zeta_k | \theta^{J}) \cdot f(\zeta_k, u_{a_k} | \theta^{f}) + \log(\gamma) J(\zeta_k | \theta^J))^2 \text{ with respect to } \theta^{J}
      \end{align*}
    \ENDFOR
  \end{algorithmic}
\end{algorithm}

\section{Experiment}

\subsection{Description of the experiment}

We consider a closed 2D environment, represented on Figure~\ref{fig:environment}. For a time $t$, the possible values of $\zeta_t$ are continuous, and the possible values of the action $a_t$ are discrete but the associated control $u_t$ is generally small.

The agent is identified by its coordinates in the plane and its orientation. This orientation influences the agent's vision, which has a visual field. At each instant, the agent can move forward by an elementary distance for the action of walking or by a greater elementary distance for the action of running. It can also rotate on itself by an angle to the left or to the right in order to change its direction and its vision.

The environment contains hidden resources necessary for the agent's survival. This can be identified for a biological agent as a source of proteins, a source of carbohydrates, water, etc... Each resource corresponds to a characteristic (dimension) of the agent's internal state. The resources are unlimited and do not appear randomly in space and remain at fixed points. This ensures that the agent has learned the configuration of its environment by returning directly to the supply points when it is in shortage. Otherwise, the agent does not need to keep track of positions and just has to move randomly until a resource is within its view-field and judges whether it is good to take in its current state. When the agent is in a circle close enough to a resource, it can choose at any time the action of consuming it in elementary quantity and stop when it wants. When a resource is within the agent view-field and close enough, it can choose to perform the succession of elementary movements leading it to this resource.

We endow the agent with physiological properties (internal state or body dynamics including fatigue and the state of immobilized sleep) and kinematic properties (distinct actions to walk, run).  The model of the internal state (the body) of the agent includes two types of fatigue, "muscle" fatigue, which depends on how far the agent has moved without resting (e.g., continuous movement), and "sleep" fatigue, if it has not recovered for too long. Splitting fatigue into such two separate terms, allowed us to reflect multiple behavioral and physiological aspects that cause natural agents (animals) to rest. Implementation of these properties is defined in the Appendix. 

At any time and in any place, the agent can choose the action of sleeping for a minimal renewable duration. This action will immobilize it for a certain duration. From a certain threshold of muscular fatigue, it is no longer possible to run, and from another higher threshold, it is also no longer possible to walk, and the agent must remain immobile to rest his muscles and reduce his muscular fatigue. Beyond a certain threshold of sleep-related fatigue, the only action that becomes possible is sleeping.

In this environment, the agent who starts with zero knowledge has the goal of finding a policy to minimize its deviation function. To do this, it will explore its environment and learn the reaction of its body to the actions it takes. Concretely, the agent, who only has access to $\zeta$, by taking actions in its environment, obtains rewards and updates its estimate of its transition function $f$, which will also allow it to update its deviation function $J$ useful to choose future actions.

The definitions of the parameters of the numerical simulation, the functions regulating the agent's body and the neural networks used are presented in the appendix.

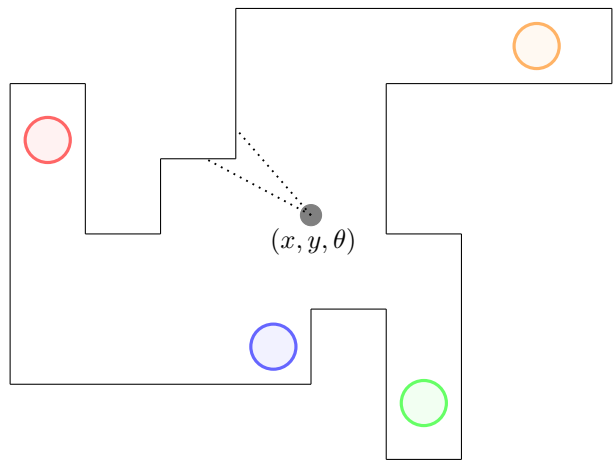
\begin{figure}
\centering
\begin{tikzpicture}

\draw (1,1) -- (1,5);
\draw (1,5) -- (2,5);
\draw (2,5) -- (2,3);
\draw (2,3) -- (3,3);
\draw (3,3) -- (3,4);
\draw (3,4) -- (4,4);
\draw (4,4) -- (4,6);
\draw (4,6) -- (9,6);
\draw (9,6) -- (9,5);
\draw (9,5) -- (6,5);
\draw (6,5) -- (6,3);
\draw (6,3) -- (7,3);
\draw (7,3) -- (7,0);
\draw (7,0) -- (6,0);
\draw (6,0) -- (6,2);
\draw (6,2) -- (5,2);
\draw (5,2) -- (5,1);
\draw (5,1) -- (1,1);

\filldraw[color=red!60, fill=red!5, very thick](1.5,4.25) circle (0.3);
\filldraw[color=blue!60, fill=blue!5, very thick](4.5,1.5) circle (0.3);
\filldraw[color=orange!60, fill=orange!5, very thick](8,5.5) circle (0.3);
\filldraw[color=green!60, fill=green!5, very thick](6.5,0.75) circle (0.3);

\filldraw [gray] (5,3.25) circle (4pt);
\draw[dotted, line width=0.25mm] (5,3.25) -- (3.6,4);
\draw[dotted, line width=0.25mm] (5,3.25) -- (4,4.4);
\node[text width=3cm] at (6,2.9) {$(x,y,\theta)$};
\end{tikzpicture}

\caption{The environment of the simulation experiment. The agent is represented by a gray point and is located by its coordinates in the plane and its orientation. The colored circles delimit the space in which it is possible to consume a resource.}
\label{fig:environment}
\end{figure}

\subsection{Results of the experiment}

After a sufficient number of iterations of the algorithm, we can empirically observe that the agent has learned well to locate resources in the environment and identify the resources it needs based on its current state.

To provide a more quantitative measure of this result, we can track the agent's internal state at the end of the training experiment. To do this, we plot the deviation function for different levels of need (see Figure~\ref{fig:results_J_function}), for example$(c_1, c_2) \mapsto J([0, 0, 0, -x_{4}^{*}, 0, 0, c_1, c_2, 0]^{T})$, which is the plot of the in-plane deviation function for the state of an agent perfectly at its homeostatic point for all but the fourth feature, the access point of which is in the green circle at the bottom right of the environment. 

We can notice in this figure that when the agent needs to consume a resource, the values of the deviation function are lower in the locations close to this resource, which indicates that the agent has learned to identify the advantageous locations in its environment according to its internal needs.

\begin{figure}
\centering
\begin{tabular}{cc}
\subfloat[]{\includegraphics[width = 2.7in]{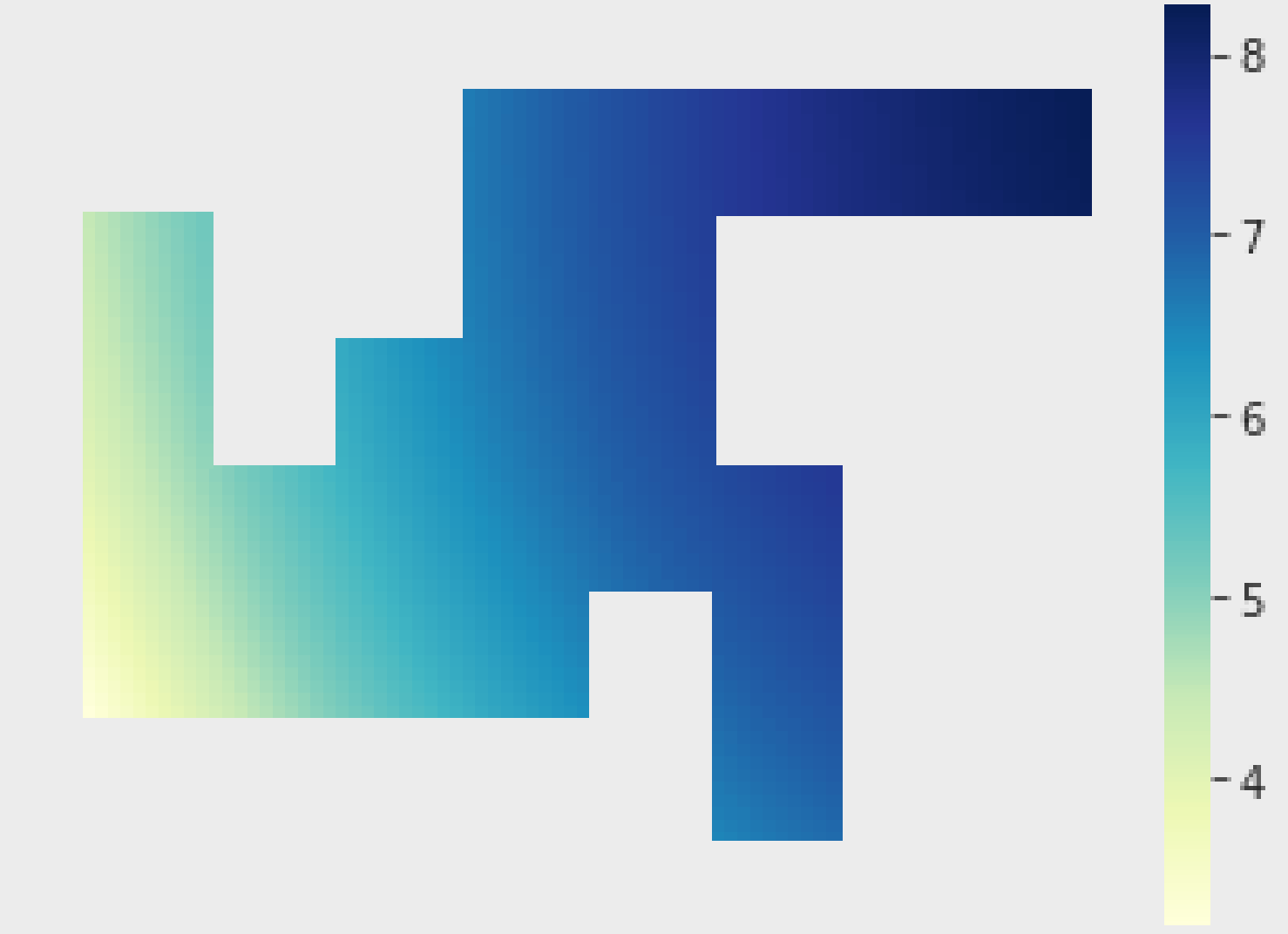}} &
\subfloat[]{\includegraphics[width = 2.7in]{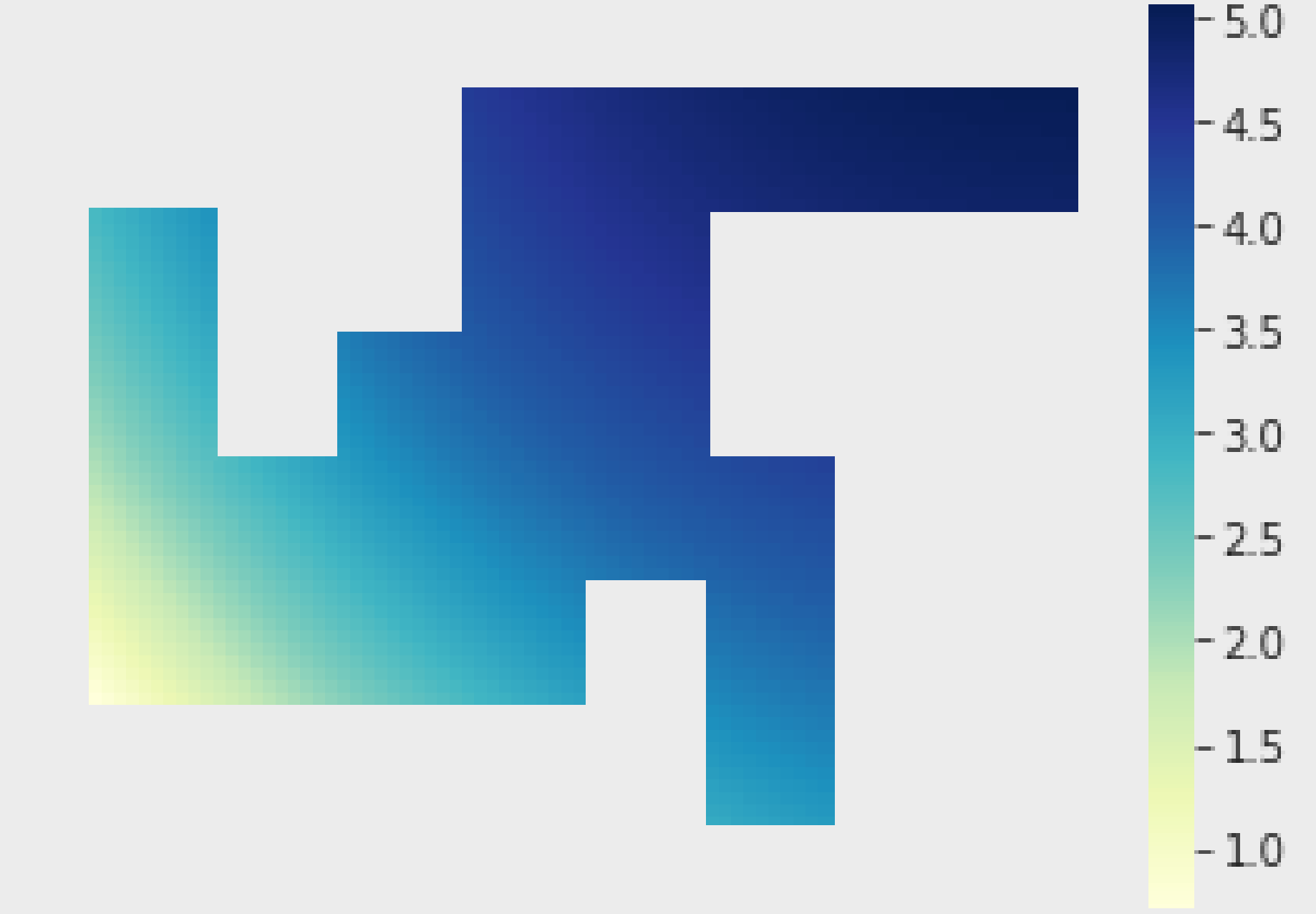}} \\
\subfloat[]{\includegraphics[width = 2.7in]{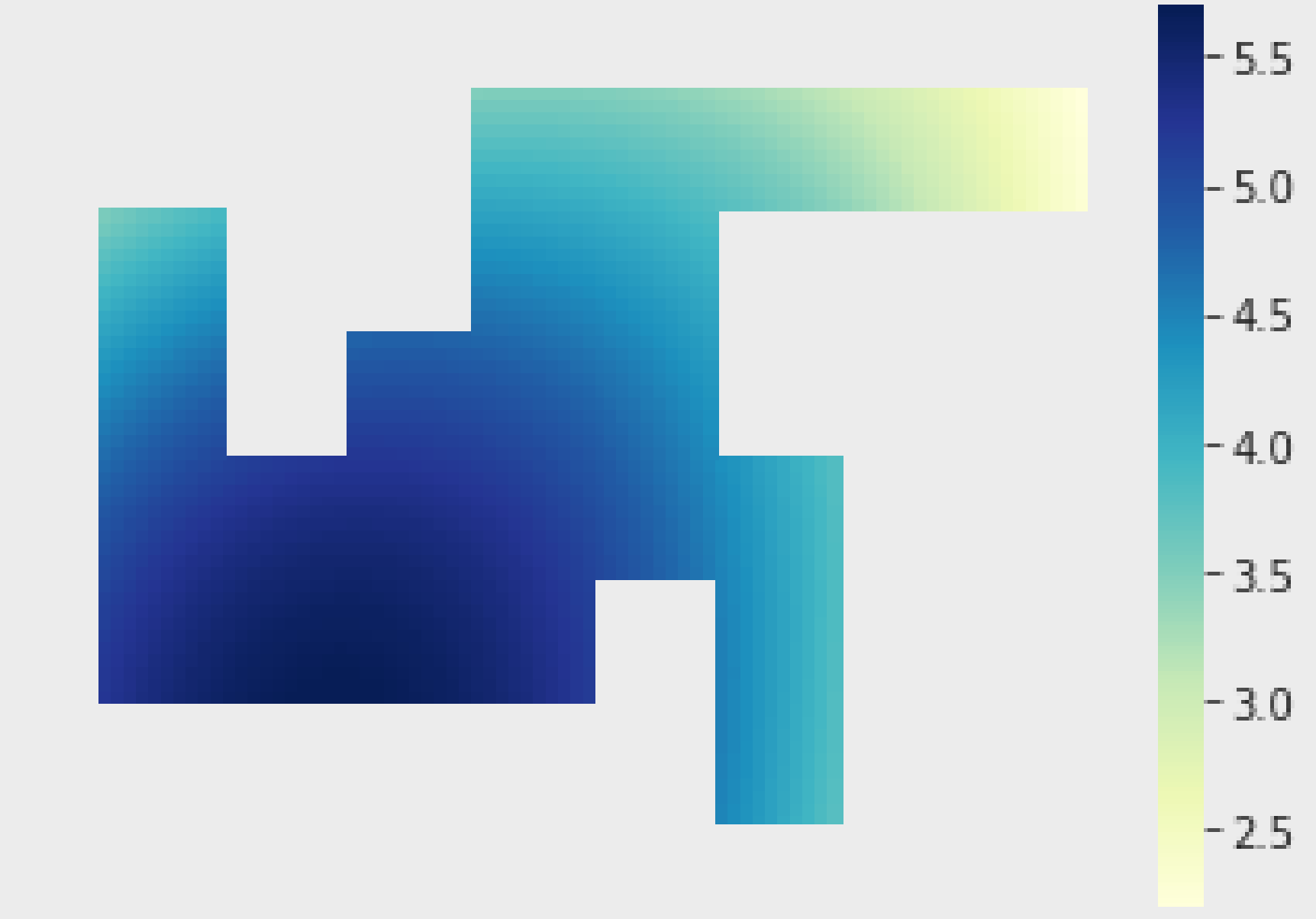}} &
\subfloat[]{\includegraphics[width = 2.7in]{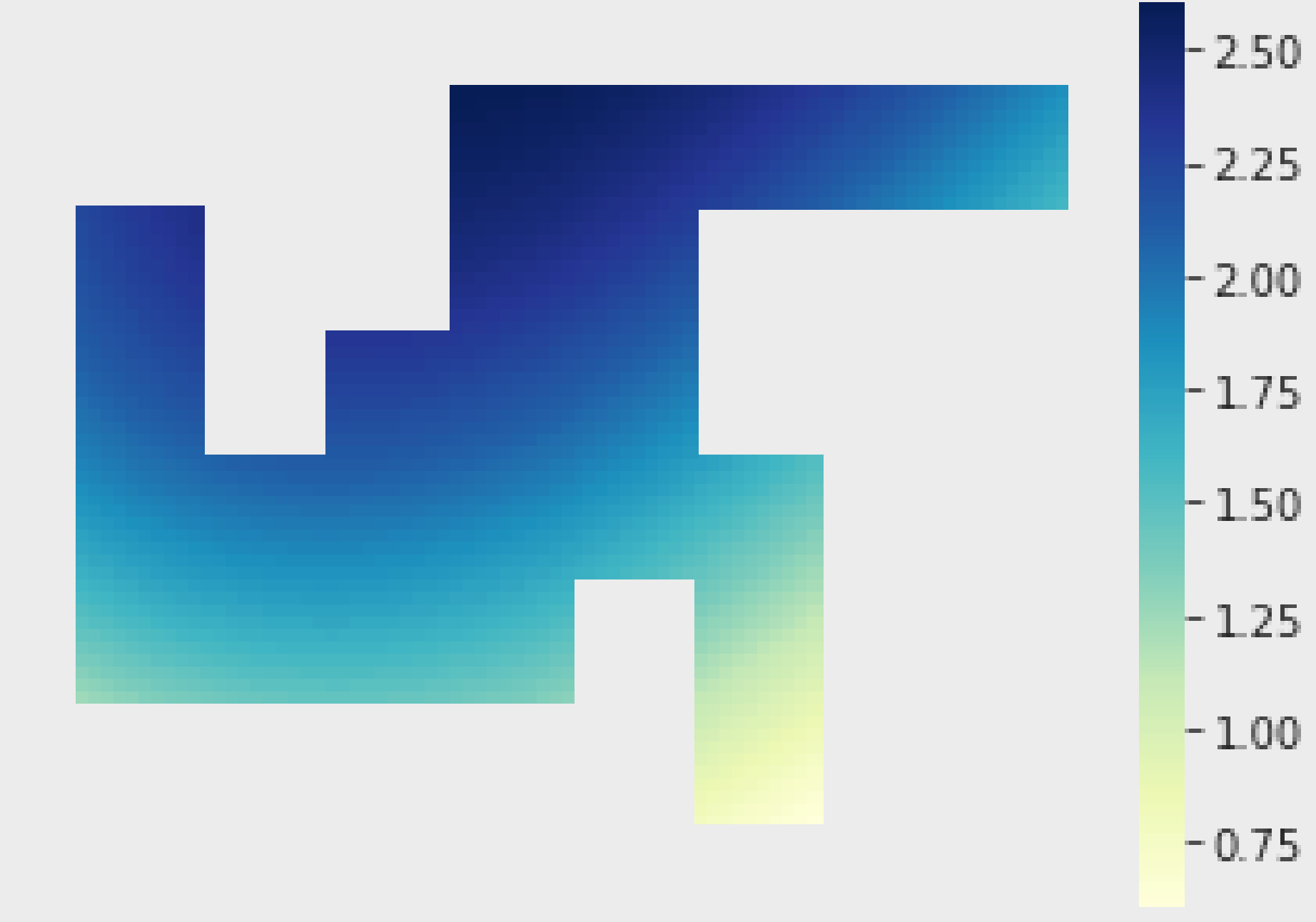}}
\end{tabular}
\caption{Plot of the agent deviation as a function of the environment coordinates at the end of the training. (a) $(c_1, c_2) \mapsto J([-x_{1}^{*}, 0, 0, 0, 0, 0, c_1, c_2, 0]^{T})$. (b) $(c_1, c_2) \mapsto J([0, -x_{2}^{*}, 0, 0, 0, 0, c_1, c_2, 0]^{T})$. (c) $(c_1, c_2) \mapsto J([0, 0, -x_{3}^{*}, 0, 0, 0, c_1, c_2, 0]^{T})$. (d) $(c_1, c_2) \mapsto J([0, 0, 0, -x_{4}^{*}, 0, 0, c_1, c_2, 0]^{T})$.}
\label{fig:results_J_function}
\end{figure}

\section{Conclusion}

The HRRL theory, a reinforcement learning extension of the classical drive reduction theory, offers a compact and simple modeling of the agent's needs and the rewards offered by an environment. In this framework, we simulated agents interacting and learning in their environment. This theory allows taking into account several types of rewards, without having to define them explicitly, to model directly primitive behaviors, and to have an intuition on more complex behaviors. We have extended the HRRL theory by providing an agent with a self-regulation mechanism of its internal state, a better consideration of temporal aspects by developing in particular a continuous time formulation of the HRRL and a learning scheme to survive in its environment. As suggested in \cite{Man}, neural networks were used to map the internal space to the external space. We showed that continuous-time HRRL is normative and naturally yields properties observed in behaviors, e.g., dose dependence of reward, modulation of reward value with deprivation (in analogy to incentive salience theory (ISR)), bounded rationality discussed in prospect theory - such as risk avoidance, loss sensitivity. We present an example of a simulation experiment, showing that our HRRL agent is able to learn to gather resources from the environment to satisfy its internal needs, also taking into account the dynamics of the internal state and the cost of actions.

However, explaining complicated behaviors is challenging because they may not be explicitly oriented to satisfy a primary life need or have a direct relationship with the body, and are mostly individual-specific. We can give as examples pro-social behaviors, self-sacrifice, desire for recognition, gambling, etc. Nevertheless, it is possible that these behaviors are somehow translated, on a small scale, into a set of characteristics that could be represented in terms of $x$ motor function variables. Indeed, \citeauthor{Juechems} argues that even non-physiological motivations can be modeled using the HRRL framework. They argue that when a human being seeks to accomplish a long-term goal, he or she plans in terms of intermediate goals to be achieved (e.g., wealth accumulation appears to be a series of financial goals to be achieved). Thus, there would be a similarity in the structure of learning when achieving a goal, whether it is a complicated goal like acquiring financial stability or a more primitive goal like having the right internal temperature. Our model is therefore essentially limited by knowledge of the human body and the structure of more abstract needs, which means that defining specific training functions for complex goals is a challenge. HRRL is also challenged by the number of internal variables that can become very large, making it difficult for the algorithm to converge and thus for the agent to learn. A research goal could be to model robots with automatically learned human characteristics, evolving and interacting together in an environment. Despite its limitations, the simplicity of our model and its ability to have an arbitrary choice of scale may have an impact on this goal. We have made simplifications on the generality of the simulation, notably on the lack of stochasticity by not putting other agents in the environment, but we could in the future make simulations with several agents including prey and predators. In a multi-agent simulation, we could also create a colony of agents in which each member implements de facto empathy for its cohorts, as suggested by \cite{Man}, e.g., by modeling an agent's welfare as a function of the weighted average of the internal states of all colony members. It would be interesting to understand how individual inferences about the collective welfare (internal states) of the social group can be learned. We also leave the door open for future research to incorporate more complicated behaviors into the model and to provide theoretical results on the stability of the agent learning algorithm.

\section*{Acknowledgement}

This work has been supported by ANR-17-EURE-1553-0017, and ANR-10-IDEX-0001-02. BSG acknowledges funding from the Basic Research Program at the National Research University Higher School of Economics (HSE University). The funders had no role in study design, data collection and analysis, decision to publish, or preparation of the manuscript.

\newpage
\nocite{*}
\bibliographystyle{abbrvnat}

\end{document}